\documentclass{article} 
\usepackage{preprint,times}

\usepackage{amsmath,amsfonts,bm}

\newcommand{\supp}{\text{supp}}
\newcommand{\triangleq}{\stackrel{\triangle}{=}}

\def\eqref#1{equation~\ref{#1}}

\def\1{\bm{1}}

\DeclareMathAlphabet{\mathsfit}{\encodingdefault}{\sfdefault}{m}{sl}
\SetMathAlphabet{\mathsfit}{bold}{\encodingdefault}{\sfdefault}{bx}{n}

\newcommand{\E}{\mathbb{E}}

\newcommand{\KL}{D_{\mathrm{KL}}}
\newcommand{\Var}{\mathrm{Var}}

\usepackage[colorlinks=true,linkcolor=blue,citecolor=blue,urlcolor=blue]{hyperref}
\usepackage{url}
\usepackage{graphicx}
\usepackage{subcaption}
\usepackage{amsthm}
\usepackage{pifont}
\usepackage{booktabs}
\usepackage{arydshln}

\newcommand{\ourapproach}{RL-PLUS}

\renewcommand{\E}{\mathbb{E}}
\renewcommand{\KL}{\mathcal{D}_{\text{KL}}}

\newcommand{\pione}{\pi_\omega}
\newcommand{\pitwo}{\pi_{\theta_\text{old}}}

\newtheorem{theorem}{Theorem}[section]
\newtheorem{lemma}[theorem]{Lemma}
\newtheorem{assumption}[theorem]{Assumption}
\theoremstyle{definition}
\newtheorem{definition}[theorem]{Definition}
\theoremstyle{remark}
\newtheorem{remark}[theorem]{Remark}

\iclrfinalcopy
\title{ 
\fontsize{15.9pt}{22pt}\selectfont RL-PLUS: Countering Capability Boundary Collapse of LLMs in Reinforcement Learning with Hybrid-policy Optimization
}

\author{Yihong Dong$^{1,2}$, Xue Jiang$^{1,2}$, Yongding Tao$^{1}$, Huanyu Liu$^{1}$, Kechi Zhang$^{1}$, Lili Mou$^{3,4}$,\\ \textbf{Rongyu Cao$^{2}$, Yingwei Ma$^{2}$, Jue Chen$^{2}$, Binhua Li$^{2}$, Zhi Jin$^{1}$, Fei Huang$^{2}$, Yongbin Li$^{2}$, Ge Li$^{1}$}\\
$^1$ School of Computer Science, Peking University \quad $^2$ Tongyi Lab, Alibaba Group\\
$^3$ Department of Computing Science, University of Alberta \quad $^4$ Canada CIFAR AI Chair\\
\texttt{dongyh@stu.pku.edu.cn} \quad \texttt{lige@pku.edu.cn}\\\
}

\begin{document}

\maketitle

\begin{abstract}

Reinforcement Learning with Verifiable Reward (RLVR) has significantly advanced the complex reasoning abilities of Large Language Models (LLMs). However, it struggles to break through the inherent capability boundaries of the base LLM, due to its essentially on-policy strategy coupled with LLM's immense action space and sparse reward. Critically, RLVR can lead to the capability boundary collapse, narrowing the LLM's problem-solving scope.
To address this problem, we propose RL-PLUS, a novel hybrid-policy optimization approach for LLMs that synergizes internal exploitation with external data to achieve stronger reasoning capabilities and surpass the boundaries of base models. 
RL-PLUS integrates two core components, i.e., Multiple Importance Sampling to address distributional mismatch from external data, and Exploration-Based Advantage Function to guide the model towards high-value, unexplored reasoning paths. 
We provide both theoretical analysis and extensive experiments to demonstrate the superiority and generalizability of our approach. Compared with existing RLVR methods, RL-PLUS achieves 1) state-of-the-art performance on six math reasoning benchmarks; 2) superior performance on six out-of-distribution reasoning tasks; 3) consistent and significant gains across diverse model families, with average relative improvements up to 69.2\%. Moreover, the analysis of Pass@k curves indicates that RL-PLUS effectively resolves the capability boundary collapse problem.\footnotetext{Work done during Yihong Dong and Xue Jiang's internship at Tongyi Lab.}\footnote{Our code is available at \url{https://github.com/YihongDong/RL-PLUS}.}

\end{abstract}

\section{Introduction}
The paradigm of Reinforcement Learning with Verifiable Reward (RLVR) has significantly propelled the improvement of reasoning performance in Large Language Models (LLMs), particularly in solving complex tasks involving math and coding \citep{openaio1, deepseekr1, Kimi_k1.5}. 
RLVR optimizes LLMs' performance via a reinforcement learning (RL) process guided by verifiable reward computation, e.g., determining whether an output matches a ground-truth math answer or passes unit tests for coding. This method enables LLMs to scale their computation at test time by extending Chain-of-Thought (CoT) processes and spontaneously exhibit sophisticated cognitive behaviors such as reflection and exploration. Thus, RLVR is believed to be a promising way for LLMs to achieve continuous self-evolution toward more powerful AI \citep{deepseekr1}.

\begin{figure*}[t!]
    \centering
    \begin{subfigure}[b]{0.72\textwidth}
        \centering
        \includegraphics[width=\textwidth]{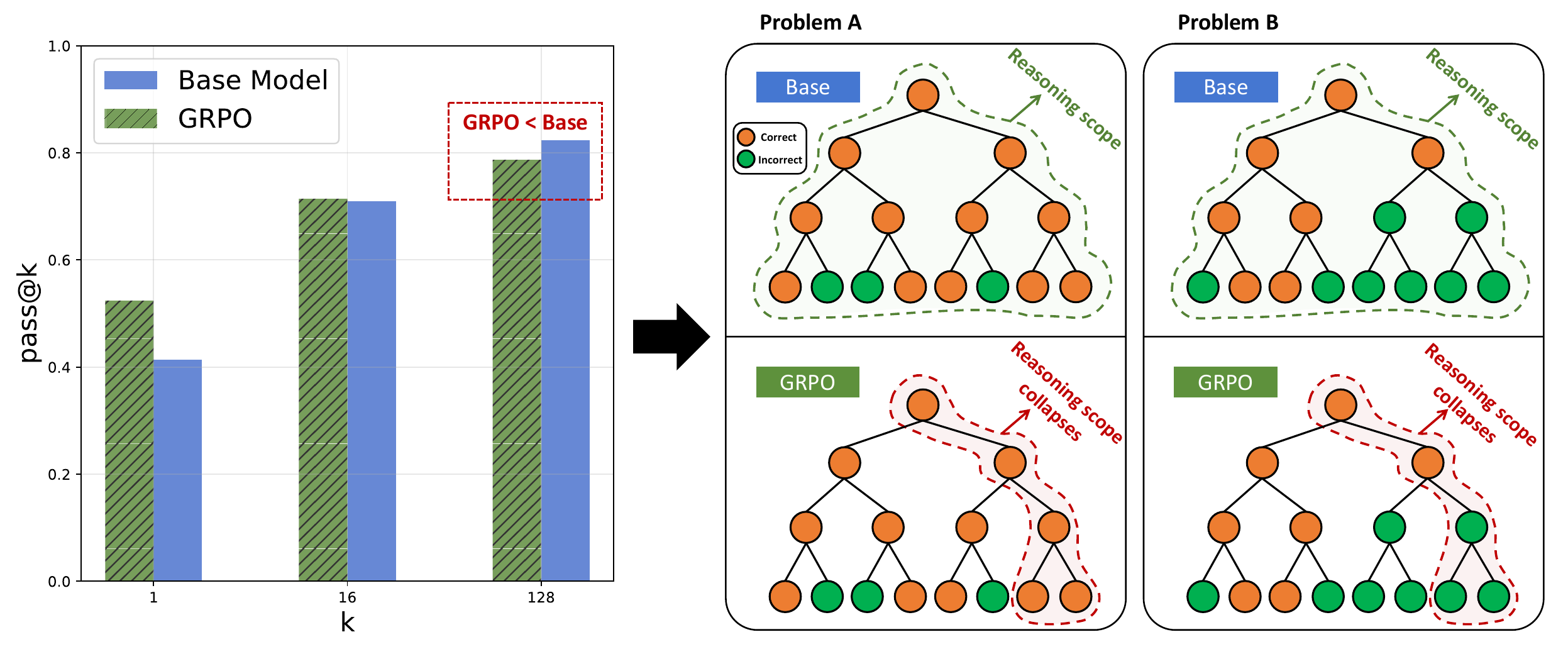}
        \caption{Collapse problem of capability boundaries in LLMs after RLVR training.}
        \label{fig:motivation1}
    \end{subfigure}
    \hfill
    \begin{subfigure}[b]{0.27\textwidth}
        \centering
        \includegraphics[width=\textwidth]{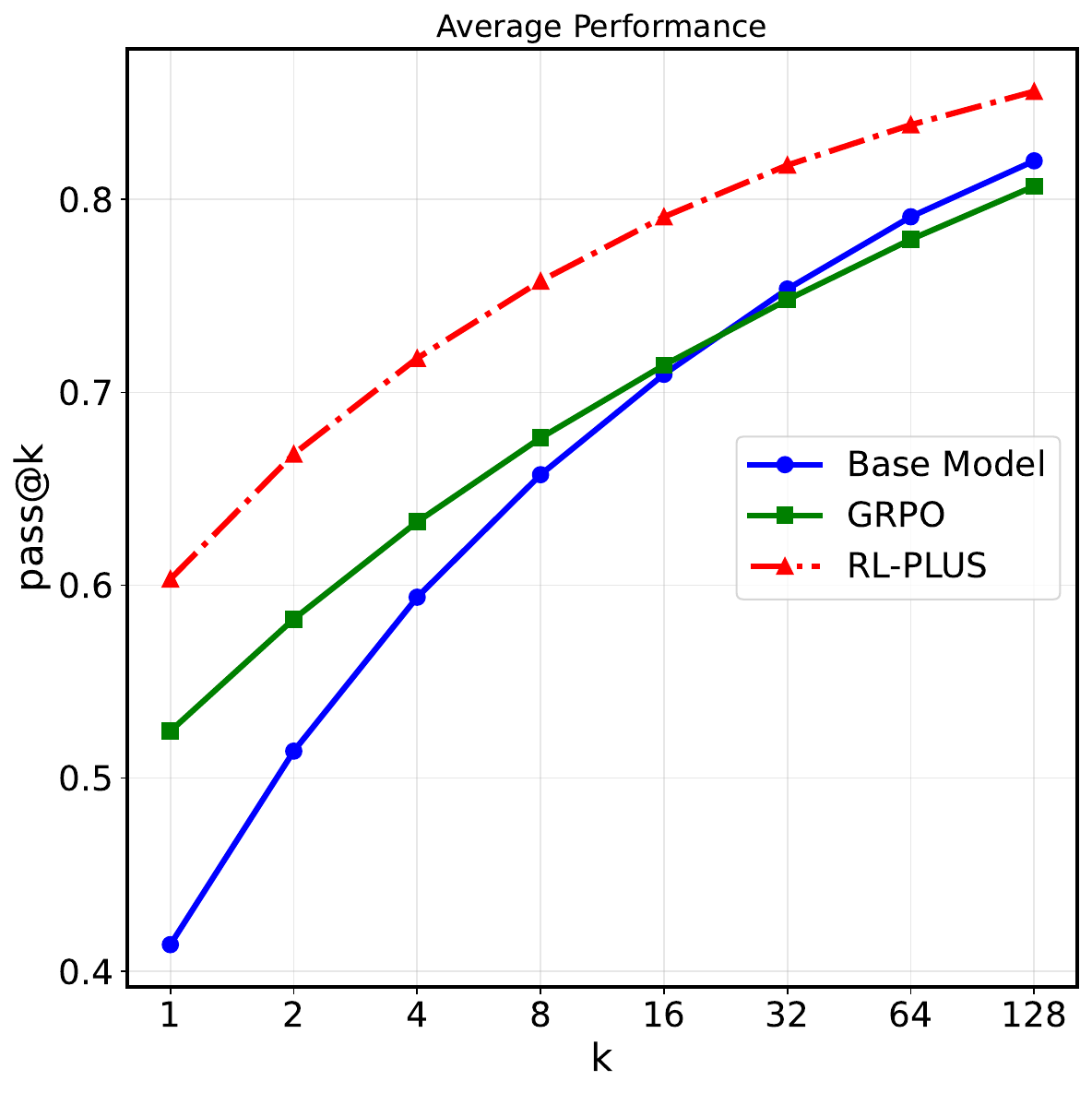}
        \caption{Benefits of our \ourapproach.}
        \label{fig:motivation2}
    \end{subfigure}
    \caption{(a) The commonly used RLVR methods can lead to the collapse problem of capability boundaries in base LLMs. (b) \ourapproach~can overcome capability boundary collapse of LLMs in RLVR, consistently showing larger pass@k than base model.
    }
    \label{fig:motivation}
\end{figure*}

Despite the empirical successes, some work \citep{ppo_meta, DeepSeekMath, ReasoningCapacity} points out that current RLVR cannot enable LLMs to acquire novel reasoning abilities, but rather simply utilize reasoning patterns already in the base model. As shown in Figure 1(a), although the pass@1 performance of RLVR-trained models surpasses that of the base model, its pass@128\footnote{The pass@k calculates the proportion of problems the model can potentially solve within a finite (k) number of attempts metric is commonly used to gauge a model's capability boundary.} is substantially lower. This trend suggests that the underlying capability distribution of the base model is broader and that existing RLVR can collapse the base model's capability boundary, thus fundamentally limiting the acquisition of new reasoning pathways.  

This limitation stems from an essential challenge when applying RLVR to LLMs: the potential solution space of LLMs is extremely immense with sparse reward that current RLVR techniques cannot effectively guide the model to explore new and unknown pathways, i.e., outward exploration. The challenge is particularly acute in long reasoning tasks where rewards are contingent upon the successful completion of an entire inferential chain. A single erroneous step can nullify the reward for the entire trajectory, thus failing to provide a positive signal for acquiring new knowledge. Consequently, the model is compelled to focus on inward exploitation, meaning that it refines and optimizes the knowledge and reasoning methods it already possesses, which results in a contraction of the model's exploratory range and a shrinking of its capabilities. This phenomenon not only prevents the model from acquiring new information or abilities that surpass its base model, but also significantly impedes any sustained enhancement of its overall performance. 

The ancient educational principle that \textit{``If one learns from others but does not think, one will be bewildered. If, on the other hand, one thinks but does not learn from others, one will be in peril''}\footnote{A principle from the philosopher and educator Confucius.} offers a crucial lens through which to view the limitations of current methodologies for enhancing LLM reasoning. Current RLVR can be viewed as the latter case, which excels at ``thinking'' through inward exploitation but demonstrates inadequate outward exploration due to its inherently on-policy strategy coupled with LLM’s immense action space and sparse reward, i.e., hard to continuous ``learning'' of new knowledge.
Conversely, approaches like Supervised Fine-Tuning (SFT) represent the former case, focusing on imitating solutions but failing to internalize the underlying reasoning principles, leading to brittleness when encountering novel problems.

This motivates us to develop novel RLVR approaches with effective external learning, but there are two key challenges that need to be addressed.
First, a distributional mismatch between the model's policy and the external data source is inevitable. Standard importance sampling corrections for RL are inadequate, i.e., employing the proxy with on-policy introduces systematic bias, whereas direct using off-policy usually suffers from high variance and bias due to their significantly divergent distributions. 
Second, there is a challenge of efficiently extracting valuable information from this external data. Models are naturally inclined to favor high-probability tokens, thus reinforcing existing knowledge. However, the key to discovering novel reasoning often lies in exploring low-probability tokens that the model would otherwise ignore.

In this paper, we propose RL-PLUS, a novel hybrid-policy optimization approach designed to synergize internal exploitation with external data during RL process.
Specifically, \ourapproach~has two core techniques. \ding{182} To resolve the issue of distributional mismatch, we employ Multiple Importance Sampling, which provides a lower bias and variance estimation of importance by combining information from multiple policies. \ding{183} To promote the discovery of new knowledge, we introduce an Exploration-Based Advantage Function, which reshapes the learning objective by prioritizing advantages for reasoning paths that are correct but are hard to explore (i.e., low probability) under the current policy. 
We also provide a theoretical analysis demonstrating that our approach achieves lower bias and variance compared with mainstream RLVR methods when leveraging external data.

Extensive experiments show the effectiveness and generalization of RL-PLUS. On six challenging math reasoning benchmarks, RL-PLUS achieves state-of-the-art (SOTA) performance, outperforming existing RLVR methods and improving upon SFT+GRPO by 5.2 average points. RL-PLUS also demonstrates superior generalization to six out-of-distribution (OOD) tasks.  RL-PLUS exhibits clear and stable improvements across diverse model families, with the average relative improvements of GRPO up to 69.2\%. Moreover, the analysis of Pass@k curves across multiple benchmarks indicates that RL-PLUS effectively transcends the inherent capability ceiling of the base model, thus addressing capability boundary collapse observed in prior RLVR approaches.

\section{Background and Related Work}
In this section, we first establish the theoretical preliminaries necessary to understand our work, and then provide a critical review of the most related work, identifying key limitations in existing methods and thereby motivating the design of our proposed RL-PLUS.

\subsection{Preliminary Knowledge}

\paragraph{LLM-based Reasoning as a Markov Decision Process.} We frame the task of generating a reasoning sequence (e.g., a solution to a math problem) as a Markov Decision Process (MDP) ~\citep{MDP}. At each timestep $t$, the state $s_t$ consists of the initial prompt $q$ concatenated with the sequence of previously generated tokens, $y_{<t}$. The action $a_t$ is the selection of the next token $y_t$ from the vocabulary. The model, or policy $\pi_\theta$, maps a state to a distribution over actions. A reward $R(q, y)$ is provided only upon completion of the entire sequence $y$. In the context of RLVR, this reward is typically sparse and binary. For example, a score is 1 if the final answer is correct and 0 otherwise. The objective is to learn a policy $\pi_\theta$ that maximizes the expected cumulative reward $J(\theta) = \mathbb{E}_{y \sim \pi_\theta}[R(q, y)]$.

\paragraph{Policy Gradient Optimization.} Policy gradient methods are the standard for optimizing LLMs in on-policy RLVR settings. Group Relative Policy Optimization (GRPO)~\citep{DeepSeekMath} shows exceptional performance in various tasks, especially to enable effective scaling within the RLVR paradigm. Compared to Proximal Policy Optimization (PPO) algorithm~\citep{PPO}, GRPO leverages group-normalized rewards to estimate advantages, eliminating the need for a value model and thereby improving computational efficiency. The standard GRPO objective is:
\begin{equation}
\label{equ:obj}
\mathcal{J}_{\text{RL}}(\theta) = \mathbb{E}_{(q,y) \sim \mathcal{D}_{\text{on}}} \left[ \sum_{t=1}^{|y|} \min \left( r_{i,t}(\theta)A_i, \text{clip}(r_{i,t}(\theta), 1-\epsilon, 1+\epsilon)A_i \right) \right] - \beta \mathcal{D}_{\text{KL}}[\pi_\theta \| \pi_{\text{ref}}]
\end{equation}
\begin{equation}
    r_{i,t}(\theta) = \frac{\pi_\theta(o_{i,t} \mid q, o_{i,<t})}{\pi_{\theta_{\text{old}}}(o_{i,t} \mid q, o_{i,<t})},
\end{equation}
\begin{equation}
    A_i = \frac{R_i - \text{mean}(\{R_1, R_2, \ldots, R_G\})}{\text{std}(\{R_1, R_2, \ldots, R_G\})},
\end{equation}

where $r_{i,t}(\theta)$ is the importance sampling ratio and $A_i$ is the estimated advantage for an on-policy trajectory $y_i$. Recent work, such as Simple-rl~\citep{Simplerl} and DAPO~\citep{DAPO}, has proposed either setting the KL coefficient $\beta$ to a very small value or omitting the KL term in Equation~\ref{equ:obj} entirely. The rationale is that during the training of a model for long CoT reasoning, the model's distribution is expected to diverge significantly from the initial policy, rendering this constraint unnecessary.

\paragraph{Evaluating Reasoning Boundaries with pass@k.} To accurately assess a model's true problem-solving capabilities, we utilize the pass@k metric~\citep{codex}. It measures the probability of obtaining at least one correct answer within $k$ independent samples for a given problem. Unlike mean accuracy (i.e., pass@1), pass@k provides a more comprehensive view of the model's reasoning potential and is critical for evaluating whether a method expands the set of solvable problems~\citep{ReasoningCapacity}.

This on-policy RLVR paradigm, while powerful, leads to two fundamental challenges when the goal is to surpass a base model's intrinsic capabilities: 1) an inability to effectively integrate novel, external knowledge due to the high variance and bias associated with off-policy data, and 2) the tendency for on-policy exploration to collapse into known, high-probability reasoning paths, thereby shrinking the model's reasoning boundary. These challenges directly motivate our approach.

\subsection{Related Work}
We position RL-PLUS by critically examining two primary lines of research: on-policy RLVR for reasoning and hybrid SFT-RL methods.

\paragraph{On-Policy RLVR and Its Intrinsic Limitations}

Reinforcement learning has become a cornerstone for enhancing LLM reasoning \citep{yue2025vapo, understanding, one_example}. Seminal works have shown that RLVR can significantly improve performance on complex reasoning tasks by rewarding correct final answers~\citep{deepseekr1, Simplerl, Open-reasoner-zero}. Subsequent research has refined this paradigm; for instance, PRIME-Zero~\citep{PRIME_Zero} uses implicit process rewards, and Oat-Zero~\citep{oat-zero} simplifies the advantage calculation in GRPO.

However, a growing body of evidence reveals a critical flaw in these on-policy methods: they primarily optimize existing knowledge rather than discovering new reasoning capabilities.
This leads to two well-documented issues. First is the Capability Boundary Collapse problem. While RLVR models often show superior pass@1 performance, their advantage diminishes as $k$ increases in pass@k evaluations, with base models eventually surpassing them~\citep{ReasoningCapacity}. This strongly suggests that RLVR refines the probability of known correct paths but fails to expand the overall set of solvable problems. Second, these methods suffer from Entropy Collapse, where policy entropy sharply decreases during training, making the model overly deterministic and hindering further exploration~\citep{entropy_mechanism}. 
This indicates that on-policy RLVR, by its nature, is prone to inward exploitation that reinforces existing biases and limits the model's potential.

\paragraph{Hybrid SFT-RL Methods}

To overcome knowledge limitations of pure RL, researchers explored hybrid methods that combine RL with SFT on external demonstration data~\citep{cai2025much}. Early approaches employed sequential, multi-stage training (SFT then RL), as seen in models like InstructGPT~\citep{InstructGPT}. While conceptually simple, this often leads to catastrophic forgetting of the SFT-learned knowledge and suffers from computational inefficiency.

More recent work has focused on unified or interleaved training frameworks. For example, ReLIFT~\citep{ReLIFT} alternates between RL and online fine-tuning on difficult problems, while LUFFY~\citep{LUFFY} selectively imitates high-quality external trajectories using a mixed policy. In another example, TAPO~\citep{TAPO} enhances RL by integrating external, high-level guidance in the form of ``thought patterns" abstracted from prior data. Other methods, such as SASR~\citep{SASR} and SuperRL~\citep{SuperRL}, employ adaptive switches to dynamically balance SFT and RL objectives based on training state. While these methods are more sophisticated, they often rely on complex, potentially unstable heuristics for balancing the two learning signals. Moreover, simply adding an SFT loss to the RL objective, as explored in ``GRPO w/ SFT Loss", can degrade performance, highlighting the difficulty of effective integration. Even advanced frameworks like UFT~\citep{UFT}, which aim to unify SFT and RL to accelerate convergence, do not explicitly address how to stabilize off-policy updates while simultaneously directing exploration towards novel solutions.

\paragraph{Motivation}

The foregoing analysis reveals persistent gaps in the related work. On-policy RLVR methods are constrained by the base model's inherent knowledge, while existing hybrid SFT-RL methods lack a principled mechanism to both stabilize learning from external, off-policy data and explicitly incentivize exploration of low-probability but correct reasoning pathways. RL-PLUS is designed to directly address these deficiencies.

\section{RL-PLUS}

RL-PLUS overcomes the LLM's capability boundaries collapse problem in RLVR by integrating externally-guided exploration with the exploitation of internal reasoning pathways.

\subsection{Mitigating Distributional Mismatch with Multiple Importance Sampling}

A central challenge in learning from a static dataset $\mathcal{D}_e = \{e_i\}_{i=1}^{N}$ is the distributional shift between the target policy $\pi_\theta$ and the unknown behavior policy $\pi_\omega$. Standard importance sampling (IS) presents a dilemma for correcting this mismatch. On-policy IS estimator, which uses a proxy like $\pi_{\theta_{\text{old}}}$ in the denominator, is systematically biased when applied to external data from $\pi_\omega$ (\textbf{Lemma~\ref{lemma:is_proxy_bias}}). Conversely, the theoretically correct off-policy estimator, using weights $r^e_t(\theta) = \frac{\pi_\theta(e_t|e_{<t})}{\pi_\omega(e_t|e_{<t})}$, suffers from support mismatch of $\pi_\theta$ (\textbf{Lemma~\ref{lemma:is_support_bias}}) and prohibitively high variance as the policies diverge (\textbf{Lemma~\ref{Variance of IS Ratio}}), which destabilizes training. This issue is compounded by the fact that $\pi_\omega$ is usually unknown, rendering direct weight computation infeasible.

To solve this, we introduce Multiple Importance Sampling to construct an estimator with lower variance and controllable bias. Instead of directly estimating $\pi_\omega$, we treat the generation of an external sample as arising from a mixture policy composed of the previous policy $\pi_{\theta_{old}}$ and the external policy $\pi_\omega$. Therefore, the Multiple Importance Sampling of each token can be defined as:
\begin{equation}
    r_{i,t}^m(\theta) = \frac{2\pi_\theta(e_{i,t} | q, e_{i,<t})}{\pi_\omega(e_{i,t} | q, e_{i,<t}) + \pi_{\theta_{old}}(e_{i,t} | q, e_{i,<t})},
\end{equation}
where $e_{i,t}$ is the $t$-th token in the external data trajectory $e_i$. It replaces the aforementioned \textit{explosive bias from poor proxy or support mismatch} with a controlled, bounded distortion error (\textbf{Remarks \ref{Controlled Bias} and \ref{Overcoming Support Mismatch}}), making the overall MIS estimator robust for stable learning from external data.
The formal denominator acts as a crucial \textit{variance guardrail}. The presence of $\pi_{\theta_{\text{old}}}$, which is intentionally kept close to $\pi_\theta$, prevents the ratio from exploding even if $\pi_\omega$ is highly dissimilar, ensuring the estimator's variance remains bounded.

\begin{theorem}[Variance Robustness of MIS]
So long as there is at least one policy in the behavior pool $\{\pi_{\beta_k}\}$ (e.g., $\pi_{\beta_k^*}$) that is a good approximation of the target policy $\pi_\theta$ (i.e., $\pi_{\beta_k^*} \approx \pi_\theta$), the variance of the MIS estimator will be low. The estimator is insensitive to other arbitrarily "bad" behavior policies in the pool. (See Proof in Appendix \ref{Variance Advantage})
\end{theorem}

A key challenge remains: the behavior policy $\pi_\omega$ is unknown. We require a robust method to estimate it. Instead of naively using a proxy, we derive an estimator for $\pi_\omega$ from a principled Bayesian perspective. We frame the estimation as a decision problem where we must balance our belief in our best available model, $\pi_{\theta_{\text{old}}}$, against a state of maximal uncertainty, represented by a non-informative uniform policy $\mathcal{U}$. This allows us to hedge against model error, leading to the following Bayes-optimal estimator.

\begin{theorem}[Bayes-Optimal Policy Estimator]
Let the model space for the unknown behavior policy $\pi_\omega$ be composed of two candidate models: (1) The specific proxy policy, $\pi_{\theta_{\text{old}}}$, representing our available, specific information. (2) A non-informative uniform policy, $\mathcal{U}(\tau)$, representing maximal uncertainty.
Let the trajectory space $\mathcal{T}$ have a finite volume $V = \int_\mathcal{T} d\tau$, such that $\mathcal{U}(\tau) = 1/V$.
Under the Principle of Indifference, we assign equal prior probabilities to these models, i.e., $P(\pi_\omega = \pi_{\theta_{\text{old}}}) = P(\pi_\omega = \mathcal{U}) = 1/2$. Then, the estimator $\hat{\pi}_\omega$ that minimizes the Bayes risk (expected L2 error) is the Bayesian model average: $\hat{\pi}_\omega^*(\tau) = \frac{1}{2} \pi_{\theta_{\text{old}}}(\tau) + \frac{1}{2} \mathcal{U}(\tau)$
(See Proof in Appendix \ref{Optimal Bayesian Estimation})
\end{theorem}

\subsection{Efficient Exploration with Exploration-Based Advantage Function}

Merely incorporating external data stably is insufficient; we must also guide the model to focus on its most valuable information, especially the "new knowledge" that the model is unlikely to discover on its own. Models tend to favor high-probability tokens, whereas novel knowledge is often embedded in correct reasoning paths that the model considers to have low probability.

To this end, we design an Exploration-Based Advantage Function, $A^c_{i,t}$, which prioritizes encouraging the model to explore reasoning steps that are correct but hard to explore, which can be defined as:

\begin{equation}
A^c_{i,t} = \frac{R_i - \text{mean}(\{R_1, R_2, \dots, R_{G}\})}{\text{std}(\{R_1, R_2, \dots, R_{G}\})} \cdot C_{i,t}    
\end{equation}

The first term is the standardized reward for all trajectories, including both internal exploration and external data, and the second term is the weight to encourage exploration. Inspired by focal loss \citep{focalloss}, we define the weight $C_{i,t}$ as:

\begin{equation}
    \label{Cgamma}
    C_{i,t} = (1 - \text{detach}(\pi_\theta(e_{i,t} | q, e_{i,<t})))^\gamma, 
\end{equation}
where $\pi_\theta(e_{i,t} | q, e_{i,<t})$ represents the model's exploration probability in the correct token $e_{i,t}$ from the external data. When it is hard to explore (i.e., $\pi_\theta$ is small), the weight $C_{i,t}$ becomes large, amplifying the advantage signal for that timestep and compelling the model to attend to this overlooked region. $\gamma$ is a hyperparameter to control $C_{i,t}$. The `detach' function is a standard operation in Torch that prevents gradients from backpropagating through the probability calculation, which enhances training stability. 

\subsection{The Composite RL-PLUS Objective}

To synergize internal exploitation $\mathcal{D}_o$ with external data $\mathcal{D}_e$, we formulate the final training objective of RL-PLUS as a composite function $\mathcal{J}_{\text{RL-PLUS}}(\theta)$, which is defined as:

\begin{align}
    \mathcal{J}_{\text{RL-PLUS}}(\theta) &= \underbrace{\mathbb{E}_{(o_i, A_{i}) \sim \mathcal{D}_o} \left[ r_{i,t}(\theta) A_{i} \right]}_{\text{Internal Exploitation (Thinking)}} + \underbrace{\mathbb{E}_{(e_i, A_{i,t}^c) \sim \mathcal{D}_e} \left[ r_{i,t}^m(\theta) A_{i,t}^c \right]}_{\text{External data for Exploration (Learning)}}
\end{align}
where the first term represents the standard policy gradient objective, which is responsible for stabilizing and improving upon the model's existing reasoning capabilities. The second term constitutes the core of our contribution, which drives the policy to external exploration. It leverages our two primary innovations: 1) Multiple Importance Sampling  $r_{i,t}^m(\theta)$, which provides a low-variance, robust mechanism for integrating external data, and 2) Exploration-Based Advantage Function $A_{i,t}^c$, which re-weights the learning signal to prioritize novel yet high-value reasoning paths. 

Moreover, we omit the clipping mechanism (e.g., $\text{clip}(r_t(\theta), 1-\epsilon, 1+\epsilon)$), which would suppress the gradient signals corresponding to highly informative, low-probability events, i.e., the ``new knowledge" we aim to acquire. By removing this constraint, RL-PLUS is empowered to take larger, more assertive optimization steps when it encounters valuable information in the external data, thus accelerating the assimilation of novel knowledge and more effectively expanding its capability boundaries in RLVR.

\begin{table*}[t]
\centering
\caption{Performance of \ourapproach~against other baselines, where the best-performing result for each benchmark is highlighted in \textbf{bold} and the base model is Qwen2.5-Math-7B for all methods. }
\label{tab:main_performance}
\resizebox{\textwidth}{!}
{
\begin{tabular}{lccccccr}
\toprule
\textbf{Method} & \textbf{AIME 24} & \textbf{AIME 25} & \textbf{AMC} & \textbf{MATH-500} & \textbf{Minerva} & \textbf{Olympiad} & \textbf{Avg.} \\
\midrule
Qwen2.5-Math-7B & 11.5 & 4.9 & 31.3 & 43.6 & 7.4 & 15.6 & 19.0 \\
Qwen2.5-Math-7B-Instruct & 12.5 & 10.2 & 48.5 & 80.4 & 32.7 & 41.0 & 37.6 \\
\midrule
SimpleRL\scriptsize{~\citep{Simplerl}} & 27.0 & 6.8 & 54.9 & 76.0 & 25.0 & 34.7 & 37.4 \\
OpenReasoner\scriptsize{~\citep{Open-reasoner-zero}} & 16.5 & 15.0 & 52.1 & 82.4 & 33.1 & 47.1 & 41.0 \\
PRIME\scriptsize{~\citep{PRIME_Zero}} & 17.0 & 12.8 & 54.0 & 81.4 & 39.0 & 40.3 & 40.7 \\
Oat\scriptsize{~\citep{oat-zero}} & 33.4 & 11.9 & 61.2 & 78.0 & 34.6 & 43.4 & 43.7 \\
\hdashline
DAPO\scriptsize{~\citep{DAPO}} & 23.4 & 15.5 & 66.3 & 86.0 & 40.1 & 49.6 & 46.8 \\
TAPO\scriptsize{~\citep{TAPO}} & 33.3 &  18.6  & 77.5 & 83.4 & 38.2 & 46.2 & 49.5    \\
LUFFY\scriptsize{~\citep{LUFFY}} & 29.4 & 23.1 & 65.6 & 87.6 & 37.5 & 57.2 & 50.1 \\
ReLIFT\scriptsize{~\citep{ReLIFT}}       & 28.4 & 21.8 & 64.3 & 86.8 & 40.1 & 54.8 & 49.4 \\
\midrule
SFT & 22.2 & 22.3 & 52.8 & 82.6 & 40.8 & 43.7 & 44.1 \\
GRPO\scriptsize{~\citep{DeepSeekMath}} & 25.1 & 15.3 & 62.0 & 84.4 & 39.3 & 46.8 & 45.5 \\
GRPO w/ SFT Loss & 19.5 & 16.4 & 49.7 & 80.4 & 34.9 & 39.4 & 40.1 \\
SFT+GRPO & 25.8 & 23.1 & 62.7 & 87.2 & 39.7 & 50.4 & 48.2 \\
\hdashline
\textbf{RL-PLUS} & \textbf{33.4} & \textbf{25.9} & \textbf{68.1} & \textbf{90.2} & \textbf{43.8} & \textbf{58.8} & \textbf{53.4} \\
\bottomrule
\end{tabular}}
\end{table*}
\section{Experimental Results}
In this section, we conduct extensive experiments to demonstrate the effectiveness and generalization of RL-PLUS. Detailed experimental setups can be found in Appendix. 

\paragraph{Performence of \ourapproach.}

As shown in Table \ref{tab:main_performance}, RL-PLUS comprehensively outperforms existing RLVR methods across all evaluated applications, achieving SOTA performance. A comparison with several straightforward baselines clearly demonstrates the benefits of RL-PLUS. SFT can be viewed as a means of learning from external knowledge, while GRPO enables the model to explore solutions on its own through reinforcement learning. The combined ``SFT+GRPO" approach yields synergistic gains, illustrating the value of integrating both external knowledge and self-exploration. However, the ``GRPO w/ SFT Loss" baseline, which simply adds an SFT loss to the RL training, shows a decline in performance. This suggests that effectively merging these two learning paradigms is a non-trivial challenge. RL-PLUS significantly improves upon ``SFT+GRPO" by an average of +5.2 points, showcasing a more potent strategy for this integration.
Furthermore, when compared to concurrent methods like LUFFY and ReLIFT, which also incorporate external examples into their training process in some form, RL-PLUS also achieves superior performance, which indicates that RL-PLUS offers a more effective way for learning from external knowledge.

\begin{table*}[h]
\centering
\caption{Out-of-Distribution performance on programming tasks (i.e., HumanEval, LiveCodeBench, LeetCode) and science QA (i.e., ARC-c, GPQA-diamond, MMLU-Pro).}
\resizebox{\textwidth}{!}
{%
\begin{tabular}{lccccccc}
\toprule
\textbf{Method} & \textbf{HumanEval} & \textbf{LeetCode} & \textbf{LiveCodeBench} & \textbf{ARC-c} & \textbf{GPQA-diamond} & \textbf{MMLU-Pro} & \textbf{Avg.} \\
\midrule
Base Model & 42.1 & 22.8 & 14.9 & 18.2 & 13.1 & 30.2 & 23.6 \\
SFT & 55.5 & 8.3 & 8.1 & 75.2 & 24.7 & 42.7 & 35.8\\
GRPO & 63.4 & 21.1 & 15.3 & 81.7 & 40.4 & 47.5 & 44.9 \\
SFT+GRPO & 59.8 & 8.34 & 9.7 & 72.4 & 24.2 & 37.7 & 35.4 \\ 
\hdashline
\textbf{RL-PLUS} & \textbf{68.3} & \textbf{27.8} & \textbf{19.2} & \textbf{82.3} & \textbf{40.4} & \textbf{54.7} & \textbf{48.8} \\
\bottomrule
\end{tabular}%
}\label{tab:OOD_results}
\end{table*}

\paragraph{Performance on OOD Tasks.}

The results on OOD tasks are presented in Table \ref{tab:OOD_results}, which show that RL-PLUS achieves substantial improvements over all baselines, including the mainstream method SFT+GRPO. It surpasses the next best baseline by an average of +3.9 points. This indicates that \text{\ourapproach} not only enhances capabilities within a specific domain but also develops more fundamental reasoning abilities that generalize to other domains.
In the domain of science QA, RL-PLUS consistently outperforms both GRPO and SFT+GRPO across all benchmarks. More notably, under a significant domain shift to programming tasks, our approach maintains its strong performance and advantage. In contrast, the performance of SFT and SFT+GRPO deteriorates significantly in this area. Considering this alongside the in-domain results from Table \ref{tab:main_performance}, a clear pattern emerges: while SFT-based methods provide a strong boost for in-domain tasks, they fail to generalize and perform worse than RL-based methods in OOD scenarios. RL-PLUS resolves this trade-off. By effectively merging the external knowledge acquisition of SFT with the robust generalization of RL, it achieves superior performance in both in-domain and out-of-distribution settings, outclassing methods reliant on either paradigm alone.

\paragraph{Training Dynamics.}

In Figure \ref{fig:training_dynamics2}, we present the training dynamics of our proposed method and baselines on various benchmarks. As illustrated, RL-PLUS consistently outperforms the alternatives in terms of test accuracy and rewards throughout the training process. Notably, RL-PLUS continues to show a clear upward trend in performance even after the baselines have plateaued. We further analyze the changes in actor entropy during training. We observe that directly incorporating external data during rollouts (the green line in Figure \ref{fig:training_dynamics2}) leads to an ``entropy explosion", causing the model's outputs to become chaotic. In contrast, the entropy of the baseline models collapses to nearly zero over the course of training, indicating a loss of exploratory capability. The entropy of RL-PLUS, however, does not diminish to zero, which suggests that our trained model retains a considerable capacity for exploration.
Prior research~\citep{entropy_mechanism} has established that policy performance is achieved at the cost of policy entropy, and the depletion of entropy marks the upper limit of performance. This implies that RL-PLUS still possesses potential for further improvement. Additionally, the response length can reflect the test-time scaling performance of a method. The steadily increasing response length of RL-PLUS is the indicator of a healthy and robust training state. In contrast, while directly incorporating external data also leads to long response lengths, its low accuracy and high policy entropy suggest that this length stems from unproductive exploration rather than meaningful reasoning.

\begin{table*}[h]
\centering
\caption{The performance of \ourapproach~based on Different LLMs.}
\label{tab:various_LLMs}
\resizebox{\textwidth}{!}
{
\begin{tabular}{lccccccr}
\toprule
\textbf{Model} & \textbf{AIME 24} & \textbf{AIME 25} & \textbf{AMC} & \textbf{MATH-500} & \textbf{Minerva} & \textbf{Olympiad} & \textbf{Avg.} \\
\midrule
\textbf{LLaMA-3.1-8B} & 4.7 & 0.4 & 18.5 & 46.4 & 19.8 & 13.2 & 17.2 \\
SFT & 2.6 & 0.9 & 29.8 & 50.0 & 21.3 & 16.9 & 20.2 \\
GRPO & 3.5 & 0.5 & 19.5 & 45.0 & 20.2 & 14.2 & 17.2 \\
RL-PLUS & \textbf{11.7} & \textbf{2.1} & \textbf{35.5} & \textbf{64.4} & \textbf{29.4} & \textbf{31.2} & \textbf{29.1} \\
\hdashline
\textbf{Deepseek-Math-7B}  & 1.1 & 0.3 & 14.5 & 40.4 & 18.8 & 10.7 & 14.3 \\
SFT & 3.8 & 0.3 & 23.3 & 51.2 & 21.3 & 19.8 & 19.9 \\ 
GRPO & 2.5 & 0.2 & 17.3 & 47.0 & 20.9 & 14.5 & 17.1 \\  
RL-PLUS & \textbf{4.1} & \textbf{0.4} & \textbf{25.0} & \textbf{54.8} & \textbf{21.7} & \textbf{21.4} & \textbf{21.3} \\

\hdashline
\textbf{Qwen2.5-Math-1.5B} & 7.2 & 3.6 & 26.4 & 28.0 & 9.6 & 21.2 & 16.0 \\
SFT & 11.7 & 13.2 & 37.8 & 70.6 & 26.8 & 31.3 & 31.9 \\
GRPO & 11.8 & 7.7 & 40.2 & 61.8 & 26.8 & 32.0 & 30.1 \\
RL-PLUS & \textbf{20.4} & \textbf{13.6} & \textbf{50.0} & \textbf{80.4} & \textbf{33.1} & \textbf{45.2} & \textbf{40.5} \\

\hdashline
\textbf{Qwen2.5-Math-7B} & 11.5 & 4.9 & 31.3 & 43.6 & 7.4 & 15.6 & 19.0 \\
SFT & 22.2 & 22.3 & 52.8 & 82.6 & 40.8 & 43.7 & 44.1 \\
GRPO & 25.1 & 15.3 & 62.0 & 84.4 & 39.3 & 46.8 & 45.5 \\
RL-PLUS & \textbf{33.4} & \textbf{25.9} & \textbf{68.1} & \textbf{90.2} & \textbf{43.8} & \textbf{58.8} & \textbf{53.4} \\
\bottomrule
\end{tabular}}
\end{table*}

\begin{figure*}[t!] 
    \centering 
    \includegraphics[width=1.01\textwidth]{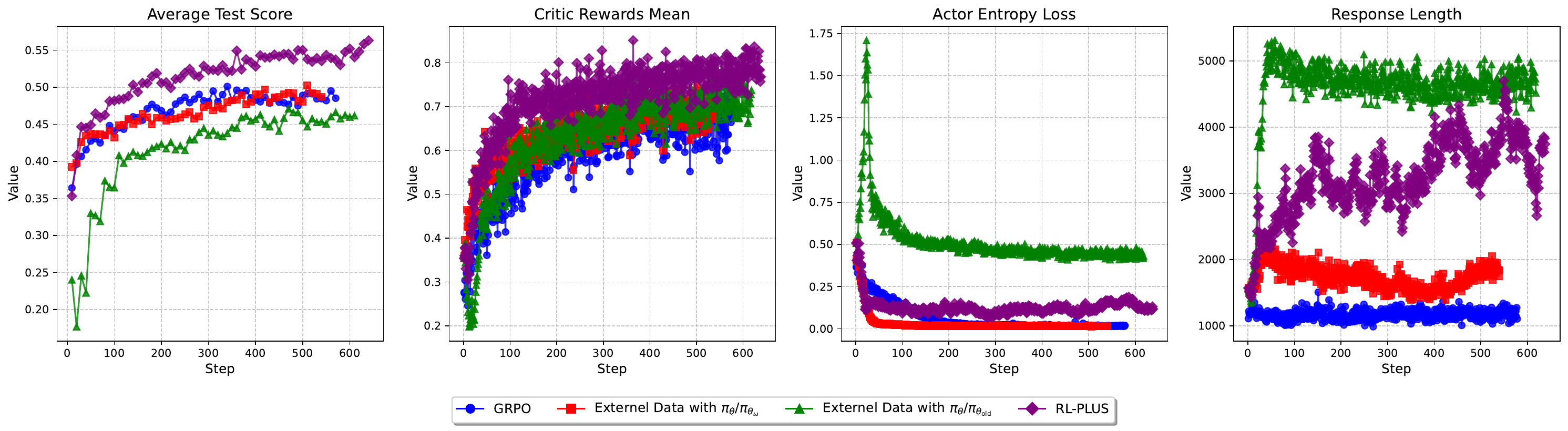} 
    \caption{Training dynamics of \ourapproach~and other baselines.} 
    \label{fig:training_dynamics2} 
\end{figure*}

\paragraph{Application on Various LLMs.}

To validate the applicability of RL-PLUS on various LLMs, we conduct experiments on several mainstream open-source LLMs, including LLaMA-3.1-8B, Deepseek-Math-7B, and the 1.5B and 7B versions of Qwen2.5-Math. The detailed results are presented in Table \ref{tab:various_LLMs}. The results indicate that RL-PLUS achieves comprehensively superior performance, regardless of the base model. Notably, on Qwen2.5-Math-7B model, RL-PLUS elevates the average score to 53.4, significantly outperforming the base model of 9.0 and other methods such as SFT of 44.1 and GRPO of 45.5. Furthermore, on LLaMA-3.1-8B, where methods like GRPO struggled to yield improvements, RL-PLUS successfully trained the model to achieve an absolute gain of 11.9 points. These findings provide evidence that RL-PLUS can consistently enhance LLMs of varying architectures and scales, significantly boosting their reasoning capabilities. 

\paragraph{Acquiring Reasoning Abilities Beyond Base Model.}

\begin{figure*}[h!]
    \centering
    \includegraphics[width=\textwidth]{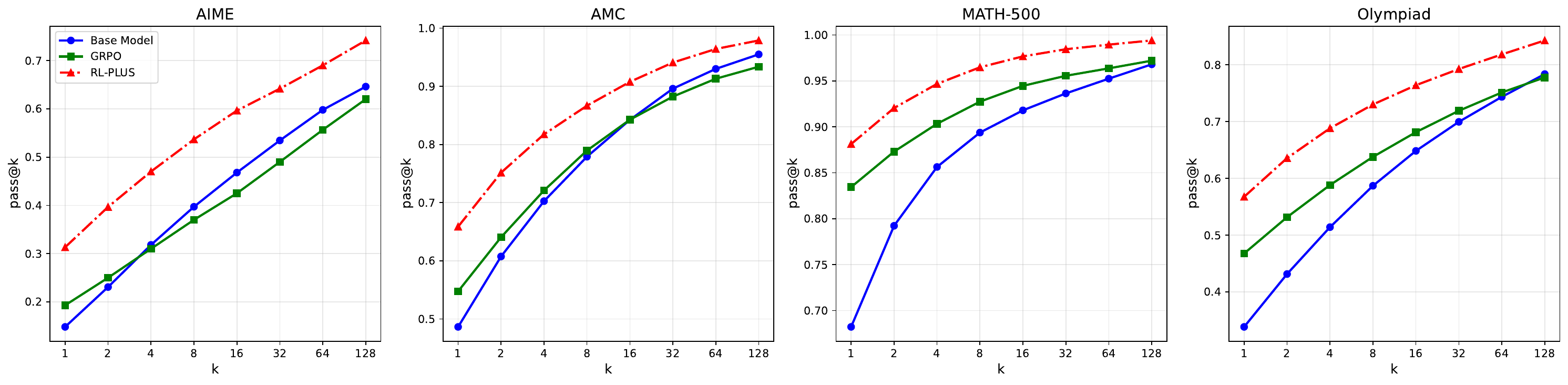}
    \caption{Pass@k curves of \ourapproach~compared with baselines across multiple benchmarks.}
    \label{fig:acquiring_reasoning_abilities}
\end{figure*}

The fundamental goal of incorporating an external policy into the RL-PLUS method is to expand the model's capability boundary by continuously introducing knowledge. Following the experimental setup of \citep{ReasoningCapacity}, we test whether RL-PLUS acquires superior reasoning abilities relative to the base model.
Figure \ref{fig:acquiring_reasoning_abilities} displays the pass@k performance curves for different methods across multiple tasks. A clear trend is observable where the performance curve of the GRPO method gradually converges with that of the base model as k increases. In some instances, GRPO's performance even drops below the base model at larger k-values, a finding consistent with that of \citep{ReasoningCapacity}.
In contrast, our approach maintains a consistent performance advantage over both the base model and GRPO as k-values increase. This sustained outperformance provides strong evidence that RL-PLUS effectively breaks through the capability boundary of the base model, rather than merely optimizing performance within its inherent ability range. On the AMC and MATH-500 tasks, the accuracy of RL-PLUS eventually plateaus because its performance is approaching the maximum possible score of 1.0.

\begin{table*}[h]
\centering
\caption{Ablation Study of RL-PLUS.}
\label{tab:ablation_variants_updated}
\resizebox{\textwidth}{!}
{
\begin{tabular}{lccccccc}
\toprule
\textbf{Method} & \textbf{AIME 24} & \textbf{AIME25} & \textbf{AMC} & \textbf{MATH-500} & \textbf{Minerva} & \textbf{Olympiad} & \textbf{Avg.} \\
\midrule
Variants with External Data\\
$\pi_\theta/\pi_{\theta_{\text{old}}}$& 19.6 & 14.8 & 55.1 & 81.0 & 33.5 & 46.2 & 41.7 \\
$\pi_\theta/\pi_{\theta_\omega}$  & 25.8 & 16.3 & 59.9 & 83.8 & 32.4 & 49.3 & 44.6 \\
$\pi_\theta/\pi_{\theta_\omega}$ with Our Policy Estimation & 26.1 & 19.2 & 62.3 & 86.8 & 38.6 & 52.0 & 47.5\\ 
\hdashline
\textbf{RL-PLUS} & \textbf{33.4} & \textbf{25.9} & \textbf{68.1} & \textbf{90.2} & \textbf{43.8} & \textbf{58.8} & \textbf{53.4} \\
- Exploration-Based Advantage Function & 28.3 & 24.1 & 67.8 & 88.8 & 40.4 & 56.0 & 50.9 \\
- Multiple Importance Sampling & 25.1 & 15.3 & 62.0 & 84.4 & 39.3 & 46.8 & 45.5 \\
\bottomrule
\end{tabular}}
\end{table*}

\paragraph{Ablation Study.}

To analyze the sources of \ourapproach's effectiveness, we conduct a series of ablation studies, with the results presented in Table \ref{tab:ablation_variants_updated}. We first ablate the two core components of our approach: Multiple Importance Sampling and Exploration-Based Advantage Function. The experimental results show that removing Exploration-Based Advantage Function causes the model's average performance to decrease from 53.4 to 50.9, which demonstrates the importance of efficient exploration for reinforcement learning. Furthermore, removing Multiple Importance Sampling leads to a more significant performance degradation, with the average score dropping substantially to 45.5, highlighting the significance of incorporating external knowledge.
Additionally, we compare our method against three naive approaches for integrating external knowledge. The first variant approximates the external policy $\pi_{\theta_\omega}$ using the old policy $\pi_{\theta_{\text{old}}}$. The second variant, an approach also seen in LUFFY~\citep{LUFFY}, approximates the external policy's probability as 1, treating it as a perfect oracle. When using our policy estimation as the external policy, i.e., the third variant, the performance improves by 2.9 points, demonstrating the effectiveness of our policy estimation. Due to the improper integration methods, these variants all show a significant performance gap compared to RL-PLUS. 

\section{Training Stability of RL-PLUS}

To validate the training stability of RL-PLUS, we extended the number of training steps on the Qwen2.5-Math-1.5B model to \textbf{over 10 times} the original setup. As shown in Figure \ref{fig:long_train}, the model's key metrics demonstrate excellent stability and continuous performance improvement as training progresses. Specifically, the Average Test Score and Critic Rewards Mean both show a steady upward trend, while the Actor Entropy Loss rapidly converges and stabilizes in a healthy, non-zero range. This reveals an ideal balance: the model's policy, while becoming more effective (i.e., exploitation), also maintains the necessary policy stochasticity for exploration, thus avoiding premature convergence to a local optimum. These results strongly demonstrate that the RL-PLUS framework possesses outstanding training stability and has the potential for further performance gains through extended training.

\begin{figure*}[h!] 
    \centering 
    \includegraphics[width=1.01\textwidth]{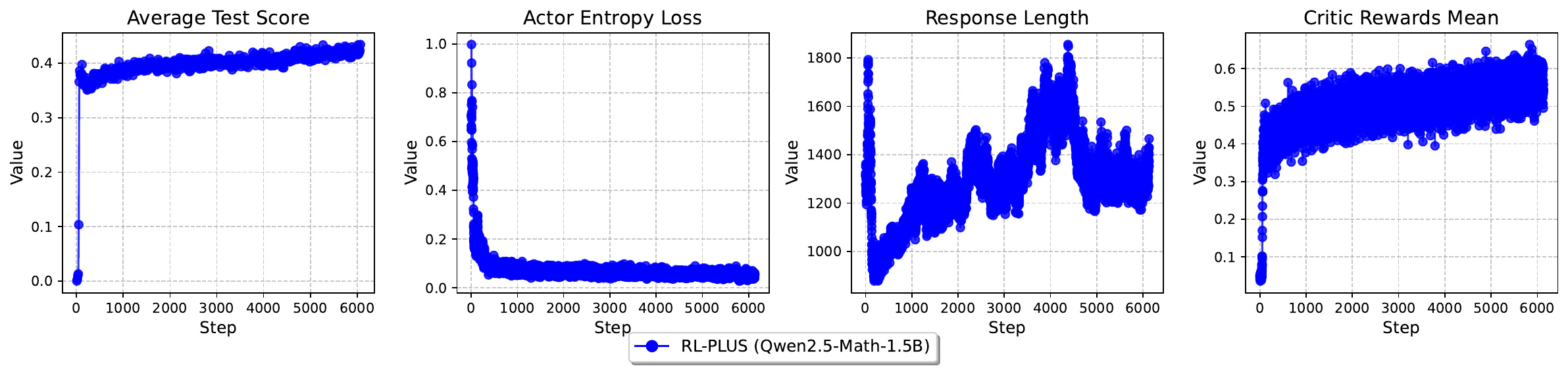} 
    \caption{Training Stability of \ourapproach.} 
    \label{fig:long_train} 
\end{figure*}

\section{Conclusion}

In this paper, we proposed RL-PLUS, a novel hybrid-policy optimization approach designed to counter the ``capability boundary collapse" observed in LLMs trained with RLVR. RL-PLUS addresses this problem by synergizing external data with internal exploitation through two core components: Multiple Importance Sampling to resolve distributional mismatch from external data, and Exploration-Based Advantage Function to incentivize the discovery of correct yet low-probability reasoning paths. We provide both theoretical analysis and extensive experiments to demonstrate the superiority and generalizability of RL-PLUS. Notably, Pass@k curves and training dynamics demonstrate that our method breaks through the reasoning capability boundary of base model, leading to further performance improvements.

\bibliography{preprint}
\bibliographystyle{preprint}

\appendix
\newpage
\onecolumn
\appendix

\section{Theoretical Analysis of Multiple Importance Sampling }
We provide a rigorous theoretical analysis of the Multiple Importance Sampling (MIS) estimator for policy optimization. First, we dissect the bias and variance issues inherent to standard Importance Sampling (IS) when using data from a single behavior policy. Subsequently, we prove that the MIS estimator is unbiased and analyze its superior variance properties. We show that MIS is robust to the inclusion of suboptimal behavior policies, establishing it as a powerful tool for integrating diverse data sources in policy optimization.

\subsection{Preliminaries and Core Assumptions}

Our analysis is based on the following standard settings and assumptions. Let the objective function be $J(\theta) = \E_{\tau \sim \pi_\theta}[R(\tau)]$, where $\tau$ represents a complete trajectory, $R(\tau)$ is its corresponding cumulative return, and $\pi_\theta$ is the target policy we aim to optimize.

\begin{assumption}[Joint Support Coverage]
The support of the target policy $\pi_\theta$ is covered by the union of the supports of all behavior policies $\{\pi_{\beta_k}\}_{k=1}^K$. Formally,
$$ \supp(\pi_\theta) \subseteq \bigcup_{k=1}^{K} \supp(\pi_{\beta_k}) $$
This assumption ensures that any trajectory possible under $\pi_\theta$ can be sampled with a non-zero probability by at least one behavior policy.
\end{assumption}

\begin{assumption}[Bounded Rewards]
The trajectory returns are bounded, i.e., for all trajectories $\tau$, there exists a constant $R_{\max}$ such that $|R(\tau)| \le R_{\max} < \infty$. This ensures that all expectations and variances are well-defined.
\end{assumption}

\subsection{Analysis of Bias and Variance in Single-Strategy Importance Sampling}

When learning from data generated by a single external behavior policy $\pi_\omega$, the standard IS estimator can suffer from bias and variance problems. We analyze three primary failure modes.

\subsubsection{Importance Sampling Estimators}

We formally define the estimators central to our analysis. We consider a dataset of $N$ trajectories.

\begin{definition}[Standard Importance Sampling (IS) Estimator]
When all data is sampled from a single behavior policy $\pi_\omega$ (i.e., $K=1, \pi_{\beta_1} = \pi_\omega$), the standard IS estimator for $J(\theta)$ is:
$$ \hat{J}_{\text{IS}}(\theta) = \frac{1}{N} \sum_{i=1}^N \frac{\pi_\theta(\tau_i)}{\pi_\omega(\tau_i)}R(\tau_i), \quad \text{where } \tau_i \sim \pi_\omega $$
\end{definition}

\begin{definition}[Proxy IS Estimator]
\label{def:proxy_is}
A biased variant of the IS estimator that uses a proxy policy $\pi_{\theta_{\text{old}}}$ in the denominator, while the data is sampled from a different policy $\pi_\omega$:
$$ \hat{J}_{\text{proxy}}(\theta) = \frac{1}{N} \sum_{i=1}^N \frac{\pi_\theta(\tau_i)}{\pi_{\theta_{\text{old}}}(\tau_i)}R(\tau_i), \quad \text{where } \tau_i \sim \pi_\omega $$
\end{definition}

\subsubsection{Bias from a Proxy}
In practice, to mitigate the high variance that occurs when the data-generating policy $\pi_\omega$ is far from the target policy $\pi_\theta$, one might be tempted to use a different policy, $\pi_{\theta_{\text{old}}}$, as the denominator for the importance ratio. This "proxy" policy is chosen to be closer to $\pi_\theta$ (e.g., a previous iterate of the policy). However, this introduces a systematic bias, as it violates the fundamental principle of importance sampling.

\begin{lemma}[Bias of the IS Estimator with a Proxy]
\label{lemma:is_proxy_bias}
Assume trajectory data $\tau_i$ is sampled from an external policy $\pi_\omega$, i.e., $\tau_i \sim \pi_\omega$. If we construct an estimator using a proxy policy $\pi_{\theta_{\text{old}}}$ in the denominator of the importance weight:
$$\hat{J}_{\text{proxy}}(\theta) = \frac{1}{N} \sum_{i=1}^N \frac{\pi_\theta(\tau_i)}{\pi_{\theta_{\text{old}}}(\tau_i)}R(\tau_i)$$
then this estimator is \textbf{biased} for the true objective $J(\theta)$ whenever the proxy policy $\pi_{\theta_{\text{old}}}$ is not identical to the true sampling policy $\pi_\omega$. The bias is given by:
\begin{equation}
    \mathcal{B}(\theta, \omega, \theta_{\text{old}}) \triangleq \E_{\pi_\omega}[\hat{J}_{\text{proxy}}(\theta)] - J(\theta) = \int \pi_\theta(\tau) R(\tau) \left( \frac{\pi_\omega(\tau)}{\pi_{\theta_{\text{old}}}(\tau)} - 1 \right) d\tau
\end{equation}
\end{lemma}

\begin{proof}
We compute the expectation of the proxy estimator $\hat{J}_{\text{proxy}}(\theta)$ under the true data distribution $\pi_\omega$. The expectation is taken with respect to $\tau \sim \pi_\omega$.
\begin{align*}
    \E_{\pi_\omega}[\hat{J}_{\text{proxy}}(\theta)] &= \E_{\tau \sim \pi_\omega} \left[ \frac{\pi_\theta(\tau)}{\pi_{\theta_{\text{old}}}(\tau)}R(\tau) \right] \\
    &= \int \pi_\omega(\tau) \frac{\pi_\theta(\tau)}{\pi_{\theta_{\text{old}}}(\tau)}R(\tau) d\tau
\end{align*}
This is the expected value that the estimator will yield. Crucially, because the sampling distribution $\pi_\omega(\tau)$ in the integral does not cancel with the denominator $\pi_{\theta_{\text{old}}}(\tau)$, this expression cannot be simplified to the true objective $J(\theta) = \int \pi_\theta(\tau)R(\tau)d\tau$.

The bias of this estimator is its expectation minus the true objective:
\begin{align*}
    \mathcal{B}(\theta, \omega, \theta_{\text{old}}) &= \E_{\pi_\omega}[\hat{J}_{\text{proxy}}(\theta)] - J(\theta) \\
    &= \int \pi_\omega(\tau) \frac{\pi_\theta(\tau)}{\pi_{\theta_{\text{old}}}(\tau)}R(\tau) d\tau - \int \pi_\theta(\tau)R(\tau) d\tau \\
    &= \int \left( \pi_\omega(\tau) \frac{\pi_\theta(\tau)}{\pi_{\theta_{\text{old}}}(\tau)}R(\tau) - \pi_\theta(\tau)R(\tau) \right) d\tau \\
    &= \int \pi_\theta(\tau) R(\tau) \left( \frac{\pi_\omega(\tau)}{\pi_{\theta_{\text{old}}}(\tau)} - 1 \right) d\tau
\end{align*}
The final expression for the bias is zero if and only if $\pi_\omega(\tau) = \pi_{\theta_{\text{old}}}(\tau)$ for all relevant trajectories. If the external data policy $\pi_\omega$ differs significantly from the proxy policy $\pi_{\theta_{\text{old}}}$, this ratio will deviate substantially from 1, leading to a large, systematic bias.
\end{proof}

\subsubsection{Bias from Support Mismatch}
Even when using the correct data-generating policy $\pi_\omega$ in the denominator, the standard IS estimator is biased if the support of $\pi_\omega$ does not fully cover the support of the target policy $\pi_\theta$.

\begin{lemma}[Bias of the Standard IS Estimator from Support Mismatch]
\label{lemma:is_support_bias}
When using data sampled from an external policy $\pi_\omega$ to estimate the objective $J(\theta)$, if the support condition $\supp(\pi_\theta) \not\subseteq \supp(\pi_\omega)$ is not met, the standard IS estimator $\hat{J}_{\text{IS}}(\theta) = \frac{1}{N} \sum_{i=1}^N \frac{\pi_\theta(\tau_i)}{\pi_\omega(\tau_i)}R(\tau_i)$ (where $\tau_i \sim \pi_\omega$) is biased. The bias relative to the true objective is:
\begin{equation}
    \mathcal{B}(\theta, \omega) \triangleq \E_{\pi_\omega}[\hat{J}_{\text{IS}}(\theta)] - J(\theta) = - \int_{\tau \in \supp(\pi_\theta) \setminus \supp(\pi_\omega)} \pi_\theta(\tau)R(\tau)d\tau
\end{equation}
\end{lemma}
\begin{proof}
The expectation of the IS estimator is calculated as follows:
\begin{align*}
    \E_{\pi_\omega}[\hat{J}_{\text{IS}}(\theta)] &= \E_{\tau \sim \pi_\omega} \left[ \frac{\pi_\theta(\tau)}{\pi_\omega(\tau)}R(\tau) \right] \\
    &= \int_{\tau \in \supp(\pi_\omega)} \pi_\omega(\tau) \frac{\pi_\theta(\tau)}{\pi_\omega(\tau)}R(\tau)d\tau \\
    &= \int_{\tau \in \supp(\pi_\omega) \cap \supp(\pi_\theta)} \pi_\theta(\tau)R(\tau)d\tau
\end{align*}
The true objective $J(\theta)$ can be decomposed over the same domains:
\begin{align*}
    J(\theta) &= \int_{\tau \in \supp(\pi_\theta)} \pi_\theta(\tau)R(\tau)d\tau \\
    &= \int_{\tau \in \supp(\pi_\theta) \cap \supp(\pi_\omega)} \pi_\theta(\tau)R(\tau)d\tau + \int_{\tau \in \supp(\pi_\theta) \setminus \supp(\pi_\omega)} \pi_\theta(\tau)R(\tau)d\tau
\end{align*}
The bias is the difference between these two quantities. This term represents the expected return from trajectories possible under $\pi_\theta$ but not under $\pi_\omega$, and it is zero if and only if the support condition holds.
\end{proof}

\subsubsection{Variance Divergence of the Importance Ratio}

\begin{lemma}[Variance of the IS Ratio]
\label{Variance of IS Ratio}
Even if the support condition is satisfied, the variance of the importance ratio $r^\omega(\tau) = \frac{\pi_\theta(\tau)}{\pi_\omega(\tau)}$ can become extremely large when the target policy $\pi_\theta$ and behavior policy $\pi_\omega$ are dissimilar. Precisely, the variance is equal to the \textbf{Chi-squared divergence} between the two policies:
$$ \Var_{\pi_\omega}(r^\omega) = \chi^2(\pi_\theta, \pi_\omega) $$
\end{lemma}

\begin{proof}
The variance of the ratio is $\Var_{\pi_\omega}(r^\omega) = \E_{\pi_\omega}[(r^\omega)^2] - (\E_{\pi_\omega}[r^\omega])^2$. Under the support coverage condition, the expectation of the ratio is $\E_{\pi_\omega}[r^\omega] = 1$. We compute the second moment:
\begin{align*}
    \E_{\pi_\omega}[(r^\omega)^2] &= \int \pi_\omega(\tau) \left(\frac{\pi_\theta(\tau)}{\pi_\omega(\tau)}\right)^2 d\tau = \int \frac{\pi_\theta(\tau)^2}{\pi_\omega(\tau)} d\tau.
\end{align*}
By noting that $\chi^2(\pi_\theta, \pi_\omega) = \int \frac{(\pi_\theta(\tau) - \pi_\omega(\tau))^2}{\pi_\omega(\tau)} d\tau = \int \frac{\pi_\theta(\tau)^2}{\pi_\omega(\tau)}d\tau - 2\int\pi_\theta(\tau)d\tau + \int\pi_\omega(\tau)d\tau = \E_{\pi_\omega}[(r^\omega)^2] - 2 + 1 = \E_{\pi_\omega}[(r^\omega)^2] - 1$, we have:
\[
\E_{\pi_\omega}[(r^\omega)^2] = \chi^2(\pi_\theta, \pi_\omega) + 1.
\]
Therefore, the variance is:
\[
\Var_{\pi_\omega}(r^\omega) = (\chi^2(\pi_\theta, \pi_\omega) + 1) - 1^2 = \chi^2(\pi_\theta, \pi_\omega).
\]
Both the $\chi^2$-divergence and the more commonly known KL-divergence ($\KL(\pi_\theta \| \pi_\omega)$) are measures of dissimilarity between distributions (both are instances of f-divergences). A large value in one typically implies a large value in the other. Therefore, as the policies diverge, there are often regions where $\pi_\theta(\tau) \gg \pi_\omega(\tau)$. In these regions, the ratio $r^\omega(\tau)$ becomes extremely large, causing the variance to explode.
\end{proof}

\subsection{Bias Advantage of the MIS Estimator}

The standard MIS estimator is proven to be unbiased. In practice, a common and highly practical scenario involves using external data collected from the behavior policy, $\pione$, which may be far from the target policy $\pi_\theta$. To stabilize estimates, one can introduce a proxy policy, $\pitwo$ (e.g., a previous iterate of $\pi_\theta$), into the denominator of the importance weight. This creates a powerful estimator that deliberately accepts a small, controlled bias in exchange for a substantial reduction in variance. 
We now formally analyze the bias advantage of this practical MIS estimator compared to the aforementioned approaches.

\begin{remark}[Controlled Bias vs. Explosive Bias of Proxy IS]
\label{Controlled Bias}
This estimator is motivated by variance reduction. While biased, its bias is far more controlled than that of the proxy estimator from Lemma~\ref{lemma:is_proxy_bias}, which uses only $\pitwo$ in the denominator. A comparison of their bias-inducing factors is revealing:
\begin{itemize}
    \item \textbf{Proxy IS Factor:} $f_{\text{proxy}}(\tau) = \frac{\pione(\tau) - \pitwo(\tau)}{\pitwo(\tau)}$
    \item \textbf{Practical MIS Factor:} $f_{\text{MIS}}(\tau) = \frac{\pione(\tau) - \pitwo(\tau)}{\pione(\tau) + \pitwo(\tau)}$
\end{itemize}
When $\pitwo(\tau) \to 0$ for a trajectory that is plausible under $\pione$, the proxy IS factor can become arbitrarily large, leading to an uncontrolled, potentially infinite bias. In contrast, the practical MIS factor is a normalized difference and is strictly bounded within $(-1, 1)$. The presence of the true sampling distribution $\pione(\tau)$ in the denominator acts as a crucial \textbf{guardrail}, preventing the weights from exploding and ensuring the bias remains bounded.
\end{remark}

\begin{remark}[Overcoming Support Mismatch]
\label{Overcoming Support Mismatch}
The practical MIS estimator also offers a robust solution to the critical problem of support mismatch (Lemma~\ref{lemma:is_support_bias}), where $\supp(\pi_\theta) \not\subseteq \supp(\pione)$. 
The practical MIS estimator mitigates this by relying on the weaker joint support assumption, $\supp(\pi_\theta) \subseteq \supp(\pione) \cup \supp(\pitwo)$. By including $\pitwo$, it explicitly covers the full support of $\pi_\theta$ and eliminates the truncation error. In its place, it introduces a \textbf{distortion error}, given by the bounded bias term derived above. In essence, this estimator replaces a potentially infinite and unrecoverable truncation error, i.e,
$$
\mathcal{B}_{\text{support}} = - \int_{\tau \in \supp(\pi_\theta) \setminus \supp(\pi_{\beta_1})} \pi_\theta(\tau)R(\tau)d\tau
$$
, with a manageable and bounded distortion error, making it a far more robust choice for real-world applications.
\end{remark}

\subsection{Variance Advantage and Robustness of the MIS Estimator}
\label{Variance Advantage}

The core advantage of MIS lies in its variance control and robustness, and we formally analyze below.

\begin{theorem}[Variance Robustness of MIS]
\label{Variance Robustness of MIS}
So long as there is at least one policy in the behavior pool $\{\pi_{\beta_k}\}$ (e.g., $\pi_{\beta_k^*}$) that is a good approximation of the target policy $\pi_\theta$ (i.e., $\pi_{\beta_k^*} \approx \pi_\theta$), the variance of the MIS estimator will be low. The estimator is insensitive to other arbitrarily "bad" behavior policies in the pool.
\end{theorem}

\begin{proof}[Proof]
We qualitatively analyze the behavior of the MIS weight $w(\tau) = \frac{\pi_\theta(\tau)}{\sum_j \alpha_j \pi_{\beta_j}(\tau)}$, whose magnitude directly drives the variance.

\textbf{Dilemma of Standard IS:} Assume we only use a "bad" policy $\pi_{\beta_m}$, for which the probability density approaches zero in some region $\mathcal{S}_{\text{bad}}$ ($\pi_{\beta_m}(\tau) \to 0$), while the target policy has non-negligible density there ($\pi_\theta(\tau) > \epsilon$). In this case, the standard IS ratio $\frac{\pi_\theta(\tau)}{\pi_{\beta_m}(\tau)}$ would diverge in $\mathcal{S}_{\text{bad}}$, causing the variance to explode.

\textbf{Advantage of MIS:} Now, we add a "good" policy $\pi_{\beta_k^*}$ to the pool, satisfying $\pi_{\beta_k^*} \approx \pi_\theta$. The denominator of the MIS weight is a mixture density: $\sum_j \alpha_j \pi_{\beta_j}(\tau)$. Even in the problematic region $\mathcal{S}_{\text{bad}}$, the denominator contains at least one term, $\alpha_{k^*} \pi_{\beta_{k^*}}(\tau) \approx \alpha_{k^*} \pi_\theta(\tau)$, which is positive and non-negligible. The MIS weight is therefore effectively bounded:
$$ w(\tau) = \frac{\pi_\theta(\tau)}{\alpha_{k^*} \pi_{\beta_{k^*}}(\tau) + \sum_{j \neq k^*} \alpha_j \pi_{\beta_j}(\tau)} \approx \frac{\pi_\theta(\tau)}{\alpha_{k^*} \pi_\theta(\tau) + \dots} \le \frac{\pi_\theta(\tau)}{\alpha_{k^*} \pi_{\beta_{k^*}}(\tau)} \approx \frac{1}{\alpha_{k^*}} = \frac{N}{n_{k^*}} $$
The weight is bounded from above by a constant that does not depend on the ratio of policies. The summation in the denominator acts as a \textbf{"variance guardrail"}, preventing the sampling deficiencies of any single policy from destabilizing the entire estimate.
\end{proof}

\begin{remark}[Practical Implications]
The robustness of MIS is especially critical when combining internal data (from an old policy $\pi_{\text{old}}$) and external data (from $\pi_\omega$). The policy $\pi_{\text{old}}$ ensures that the KL-divergence from the current policy $\pi_\theta$ is kept within a controllable range. This ensures that there is always a "good" policy in the pool. Therefore, even if the external policy $\pi_\omega$ is far from $\pi_\theta$, the MIS estimator can stabilize the variance through the presence of $\pi_{\text{old}}$. MIS achieves a "soft", unbiased form of variance control by mixing policy densities in the denominator. This adaptive weighting mechanism makes MIS a theoretically sound and highly effective choice for integrating heterogeneous data sources in policy optimization.
\end{remark}

\subsection{Optimal Bayesian Estimation of the Behavior Policy under Model Uncertainty}
\label{Optimal Bayesian Estimation}

We need a method to construct a robust estimator for $\pi_\omega$ that acknowledges our uncertainty. We propose a principled approach based on Bayesian decision theory to derive an optimal estimator for $\pi_\omega$ that explicitly balances our belief in the proxy model $\pi_{\theta_{\text{old}}}$ with a model of maximal uncertainty.

We frame the task of selecting an estimator $\hat{\pi}_\omega$ as a Bayesian decision problem.
\begin{itemize}
    \item \textbf{State of Nature}: The true, unknown behavior policy $\pi_\omega$.
    \item \textbf{Action}: Our choice of an estimator $\hat{\pi}_\omega$ for $\pi_\omega$.
    \item \textbf{Model Space $\mathcal{M}$}: The set of candidate models for $\pi_\omega$. Given our limited knowledge, we define a minimal, discrete model space that captures the dichotomy between our specific knowledge and our uncertainty.
    \item \textbf{Loss Function $L(\hat{\pi}_\omega, \pi_\omega)$}: A function that quantifies the error of our estimator. A standard choice is the squared L2-error, $L(\hat{\pi}_\omega, \pi_\omega) = \int \left( \hat{\pi}_\omega(\tau) - \pi_\omega(\tau) \right)^2 d\tau$.
\end{itemize}
Our goal is to find the estimator $\hat{\pi}_\omega$ that minimizes the \textbf{Bayes risk}, which is the expected loss with respect to our prior beliefs about the state of nature.

\begin{theorem}[Bayes-Optimal Policy Estimator]
\label{thm:bayes_optimal_estimator}
Let the model space for the unknown behavior policy $\pi_\omega$ be composed of two candidate models:
\begin{itemize}
    \item The specific proxy policy, $\pi_{\theta_{\text{old}}}$, representing our available, specific information.
    \item A non-informative uniform policy, $\mathcal{U}(\tau)$, representing maximal uncertainty.
\end{itemize}
Let the trajectory space $\mathcal{T}$ have a finite volume $V = \int_\mathcal{T} d\tau$, such that $\mathcal{U}(\tau) = 1/V$.
Under the \textbf{Principle of Indifference}, we assign equal prior probabilities to these models, i.e., $P(\pi_\omega = \pi_{\theta_{\text{old}}}) = P(\pi_\omega = \mathcal{U}) = 1/2$. Then, the estimator $\hat{\pi}_\omega$ that minimizes the Bayes risk (expected L2 error) is the Bayesian model average:
$$\hat{\pi}_\omega^*(\tau) = \frac{1}{2} \pi_{\theta_{\text{old}}}(\tau) + \frac{1}{2} \mathcal{U}(\tau)$$
\end{theorem}

\begin{proof}
The Bayes risk of an estimator $\hat{\pi}_\omega$ is the expectation of the loss function over the prior distribution of $\pi_\omega$:
\begin{align*}
    R(\hat{\pi}_\omega) &= \E_{\pi_\omega} \left[ L(\hat{\pi}_\omega, \pi_\omega) \right] \\
    &= \sum_{\pi' \in \{\pi_{\theta_{\text{old}}}, \mathcal{U}\}} L(\hat{\pi}_\omega, \pi') P(\pi_\omega = \pi') \\
    &= \frac{1}{2} \int \left( \hat{\pi}_\omega(\tau) - \pi_{\theta_{\text{old}}}(\tau) \right)^2 d\tau + \frac{1}{2} \int \left( \hat{\pi}_\omega(\tau) - \mathcal{U}(\tau) \right)^2 d\tau
\end{align*}
To find the optimal estimator $\hat{\pi}_\omega^*$ that minimizes this risk, we can use the calculus of variations or simply note that the integrand is a sum of squared errors, which is minimized point-wise. For any given trajectory $\tau$, we seek to minimize:
$$f(\hat{\pi}_\omega(\tau)) = \left( \hat{\pi}_\omega(\tau) - \pi_{\theta_{\text{old}}}(\tau) \right)^2 + \left( \hat{\pi}_\omega(\tau) - \mathcal{U}(\tau) \right)^2$$
This is a simple quadratic function of the scalar value $\hat{\pi}_\omega(\tau)$. We find the minimum by taking the derivative with respect to $\hat{\pi}_\omega(\tau)$ and setting it to zero:
\begin{align*}
    \frac{\partial f}{\partial \hat{\pi}_\omega(\tau)} &= 2 \left( \hat{\pi}_\omega(\tau) - \pi_{\theta_{\text{old}}}(\tau) \right) + 2 \left( \hat{\pi}_\omega(\tau) - \mathcal{U}(\tau) \right) = 0 \\
    2\hat{\pi}_\omega(\tau) - \pi_{\theta_{\text{old}}}(\tau) - \mathcal{U}(\tau) &= 0 \\
    \hat{\pi}_\omega(\tau) &= \frac{1}{2} \left( \pi_{\theta_{\text{old}}}(\tau) + \mathcal{U}(\tau) \right)
\end{align*}
This result gives the point-wise minimizer. Integrating over all $\tau$ confirms that the optimal estimator function is:
$$\hat{\pi}_\omega^*(\tau) = \frac{1}{2} \pi_{\theta_{\text{old}}}(\tau) + \frac{1}{2} \mathcal{U}(\tau)$$
This estimator is known as the \textbf{Bayes estimator} under quadratic loss for this prior. It is optimal in the sense that no other estimator has a lower expected error, given our stated beliefs about the possible models for $\pi_\omega$. It is straightforward to verify that $\hat{\pi}_\omega^*(\tau)$ is a valid probability density function, as $\int \hat{\pi}_\omega^*(\tau)d\tau = \frac{1}{2}\int\pi_{\theta_{\text{old}}}(\tau)d\tau + \frac{1}{2}\int\mathcal{U}(\tau)d\tau = \frac{1}{2}(1) + \frac{1}{2}(1) = 1$.
\end{proof}

\begin{assumption}[Unit-Volume Trajectory Space]
For analytical tractability, we assume the trajectory space $\mathcal{T}$ is normalized to have unit volume, i.e., $\int_{\mathcal{T}} d\tau = 1$. Under this assumption, the maximum-entropy (uniform) distribution is $\mathcal{U}(\tau) = 1$ for all $\tau \in \mathcal{T}$.
\end{assumption}

\begin{remark}[Robustness and Connection to Regularization]
Theorem \ref{thm:bayes_optimal_estimator} provides a rigorous justification for what is, in essence, a form of regularization. The resulting estimator $\hat{\pi}_\omega^*$ is a mixture model that hedges against the deficiencies of $\pi_{\theta_{\text{old}}}$. The uniform component $\mathcal{U}(\tau)$ acts as a \textbf{“safety net”} or a \textbf{“defensive distribution”}. By ensuring that $\hat{\pi}_\omega^*(\tau) \ge \frac{1}{2V} > 0$ for all $\tau$, it guarantees that the importance sampling ratio's denominator is strictly positive and bounded away from zero. This prevents the variance of the importance weights from exploding, a critical property for stable off-policy learning.

The assumption $P=1/2$ reflects a state of maximal ambiguity between the specific information we have ($\pi_{\theta_{\text{old}}}$) and the general uncertainty we face ($\mathcal{U}$). It is the most conservative and robust choice when we cannot quantify our confidence in $\pi_{\theta_{\text{old}}}$. Thus, forming the estimator as their mean is the theoretically optimal strategy to navigate this uncertainty.
\end{remark}

\section{Theoretical Analysis of the Exploration-Based Advantage}

We provide a theoretical justification for the proposed Exploration-Based Advantage function. We prove that it adaptively focuses the policy gradient updates on high-value, hard-to-explore actions.

\subsection{Gradient Analysis}

We now analyze the effect of this advantage function on the policy gradient.

\begin{lemma}[Gradient Contribution of a Single Timestep]
\label{lem:gradient_contribution}
The gradient update for the policy parameters $\theta$ induced by the action $e_{i,t}$ from a correct, high-reward trajectory $i$ is given by:
$$ \Delta \theta_{i,t} \propto \nabla_\theta \log \pi_\theta(e_{i,t} | q, e_{i,<t}) \cdot A_{i} \cdot \left(1 - \pi_\theta(e_{i,t} | q, e_{i,<t})\right)^\gamma $$
\end{lemma}

\begin{proof}
The gradient update for the policy objective at timestep $t$ is proportional to $\nabla_\theta \log \pi_\theta(e_{i,t} | \dots) \cdot A^c_{i,t}$. Substituting the definition of $A^c_{i,t}$, we have:
$$ \Delta \theta_{i,t} \propto \nabla_\theta \log \pi_\theta(e_{i,t} | \dots) \cdot A_{i} \cdot C_{i,t}(\theta) $$
By the definition of $C_{i,t}(\theta)$ and the properties of the $\text{detach}$ operator, the term $C_{i,t}(\theta)$ is treated as a scalar weight during backpropagation. Substituting its definition yields the result.
\end{proof}

This lemma establishes the precise form of the gradient update. We now prove our main result: that this form adaptively focuses learning.

\begin{theorem}[Adaptive Gradient Focusing]
\label{thm:adaptive_focusing}
Given a high-reward trajectory where $A_{i} > 0$, the gradient magnitude of the update induced by $A^c_{i,t}$ is inversely related to the policy's confidence $\pi_\theta(e_{i,t}|\dots)$. The update is amplified for "hard" (low-probability) actions and suppressed for "easy" (high-probability) actions.
\end{theorem}

\begin{proof}
We analyze the asymptotic behavior of the scaling factor on the gradient, based on Lemma \ref{lem:gradient_contribution}. Let $p_t = \pi_\theta(e_{i,t} | \dots)$ denote the policy's probability for the correct action at time step $t$. The gradient is scaled by the factor $A_{i} \cdot (1 - p_t)^\gamma$. We consider two cases for the value of $p_t$.

\textbf{Case 1: Hard-to-Explore Correct Action.}
In this case, the policy assigns a low probability to the correct action, i.e., $p_t \to 0$. The exploration weight becomes:
$$ \lim_{p_t \to 0} C_{i,t}(\theta) = \lim_{p_t \to 0} (1 - p_t)^\gamma = 1 $$
The resulting gradient update, $\Delta \theta_{i,t} \propto \nabla_\theta \log p_t \cdot A_{i}$, retains its full magnitude. The learning signal from this valuable, unexplored action is preserved.

\textbf{Case 2: Easy-to-Explore Correct Action.}
In this case, the policy is already confident about the correct action, i.e., $p_t \to 1$. The exploration weight becomes:
$$ \lim_{p_t \to 1} C_{i,t}(\theta) = \lim_{p_t \to 1} (1 - p_t)^\gamma = 0 $$
The resulting gradient update vanishes: $\Delta \theta_{i,t} \to 0$. The model effectively ignores updates from examples it has already mastered.

\textbf{Conclusion.}
It demonstrates that the optimization process is focused on the gradients from actions where the policy is incorrect or uncertain, thereby prioritizing the learning of new knowledge. This proves that the advantage function leads to adaptive gradient focusing.
\end{proof}

\begin{figure*}[h!] 
    \centering 
    \includegraphics[width=1.01\textwidth]{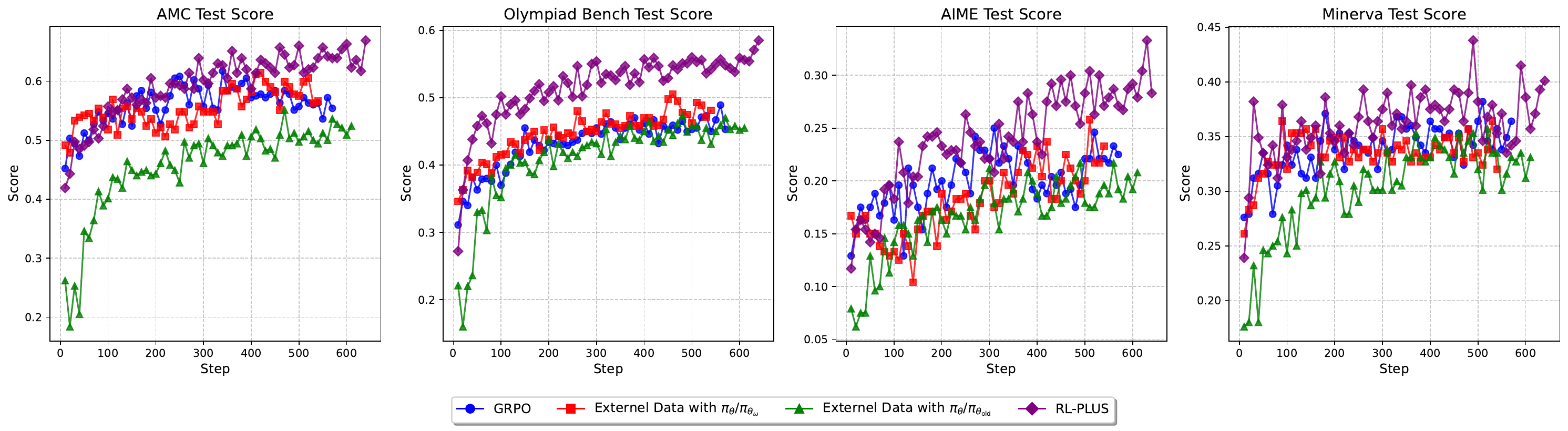} 
    \caption{Detailed Training dynamics of \ourapproach~and other baselines.} 
    \label{fig:training_dynamics1} 
\end{figure*}

\section{Effect of hyperparameter $\gamma$}

Our systematic investigation into the hyperparameter $\gamma$ in RL-PLUS, illustrated in Figure \ref{fig:hyperparam}, reveals two key findings. First, the model demonstrates considerable robustness, as its performance fluctuates only small across the tested range of $\gamma$. Second, RL-PLUS consistently surpasses the GRPO baseline across all math reasoning benchmarks, irrespective of the specific value of $\gamma$. Further analysis highlights a distinct trend: the model uniformly achieves its peak performance when $\gamma$=0.5. This optimal value holds not only for the Average test score but also across all individual benchmarks, including AMC, Olympiad, AIME, Minerva, and Math. This suggests that an intermediate value for $\gamma$ strikes an effective balance in the model's learning process. While the model is not highly sensitive to this parameter, the clear peak establishes $\gamma$=0.5 as a strong default, and there is still potential room for improvement with fine-grained tuning of $\gamma$.

\begin{figure*}[h!] 
    \centering 
    \includegraphics[width=1.01\textwidth]{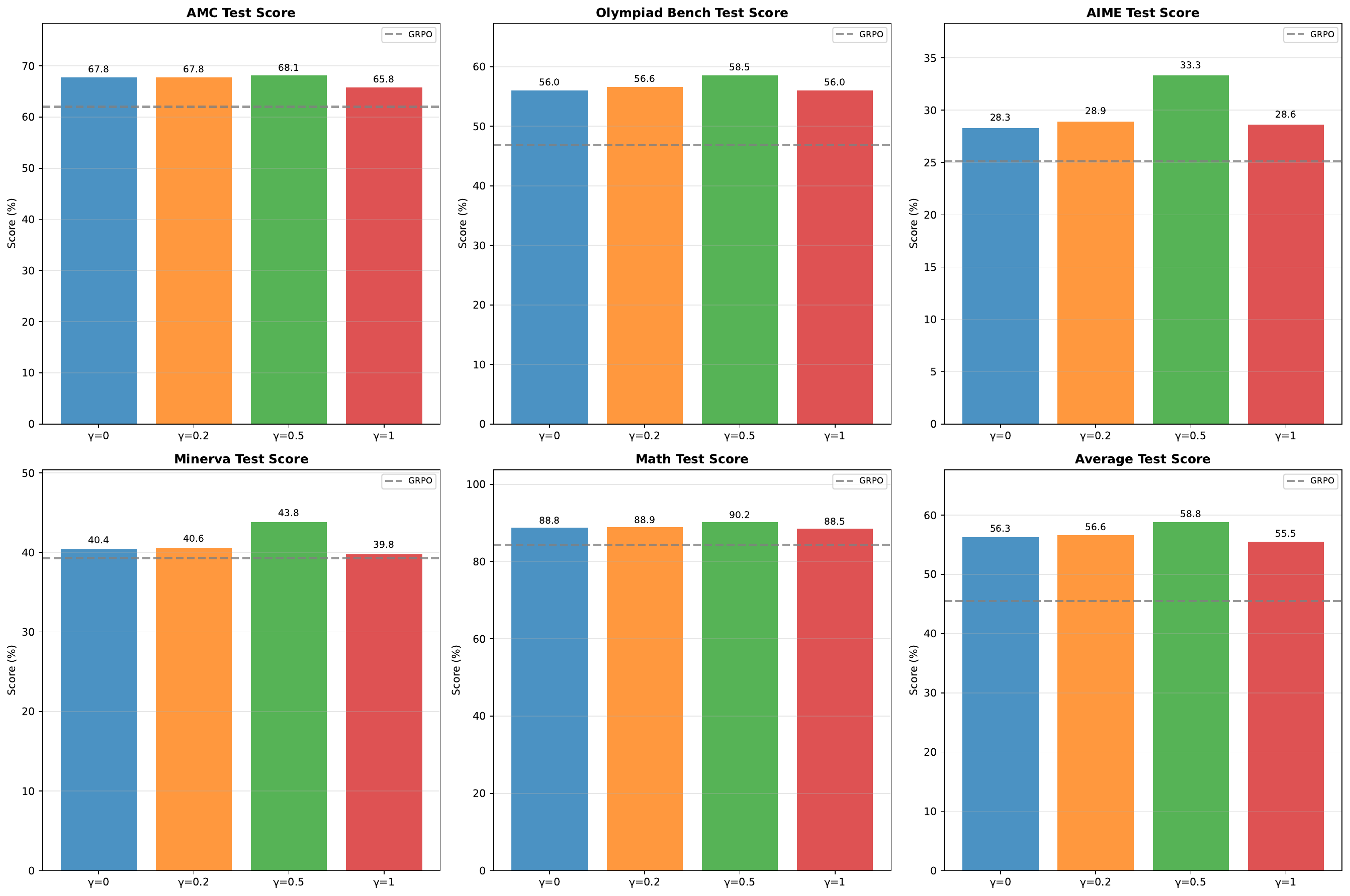} 
    \caption{Effect of hyperparameter $\gamma$ in \ourapproach.} 
    \label{fig:hyperparam} 
\end{figure*}

\section{Case Study}

\begin{figure*}[t!] 
    \centering 
    \includegraphics[width=1\textwidth]{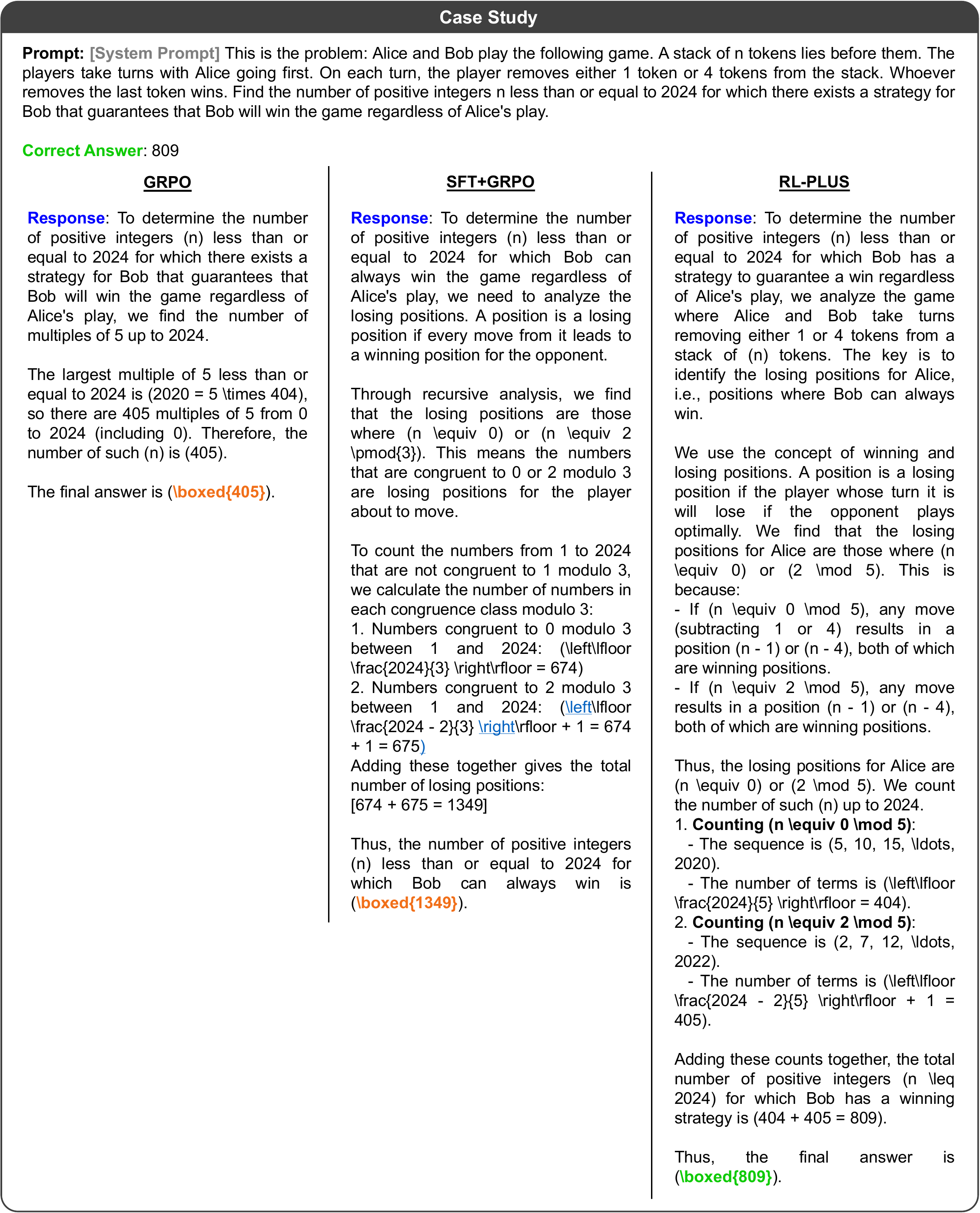} 
    \caption{A case of \ourapproach~compared with baselines GRPO and SFT+GRPO.} 
    \label{fig:case} 
\end{figure*}

Figure \ref{fig:case} presents a typical case study that visually contrasts the performance of RL-PLUS with the baseline methods, GRPO and SFT+GRPO. In this case, RL-PLUS demonstrates a significant advantage in both logical rigor and computational precision. Specifically, GRPO, while touching upon a part of the core issue by identifying `multiples of 5' as part of the losing positions, demonstrates an incomplete understanding. It fails to identify the other critical condition, thus arriving at an incorrect conclusion. SFT+GRPO's approach is fundamentally flawed. It completely misinterprets the game-theoretic model of the problem, erroneously applying an irrelevant 'modulo 3' logic, causing its reasoning to be incorrect from the outset. The performance of RL-PLUS is exemplary. It begins by accurately identifying the problem as a game of finding P-positions (second-player winning positions). Subsequently, through deductive reasoning, it successfully derives the complete pattern for the set of losing positions: when $n \equiv 0$ or $2\ mod\ 5$. Finally, it proceeds with a clear, step-by-step calculation for both conditions, sums them accurately, and arrives at the correct answer. This case provides compelling evidence that RL-PLUS possesses a more profound and comprehensive multi-step reasoning capability.

\section{Experimental Setup}

\paragraph{Training Details.}
All experiments are conducted on 8 NVIDIA A100 80G GPUs. By default, we use the Qwen2.5-Math-7B model \citep{yang2024qwen2} as the base model in our experiments. 
For our training, we use a subset of OpenR1-Math-220k \citep{openr1-math}, which contains 45,000 prompts with correct reasoning trajectories annotated by Deepseek-R1, and change the rope theta of Qwen2.5-Math-7B from 10000 to 40000 and extend the window size to 16384, following previous work \citep{LUFFY}. 
In implementing the RL algorithm, we leverage the VeRL framework \citep{sheng2024hybridflow}. We set the batch size to 128, the mini-batch size to 64, and the maximum training epoch to 2. For each problem, we use 8 rollout trajectories, with a maximum response length of 8192 tokens. For our approach, one of the model-generated rollouts is replaced with a correct reasoning trajectory from the training dataset. It is important to note that we ensure all other RL algorithms maintain the same parameter settings as RL-PLUS to guarantee a fair comparison. For the hyperparameter $\gamma$, we set it to 0.5 in all experiments by default. To validate the applicability of RL-PLUS on various base LLMs, we additionally extend RL-PLUS to other base models, including LLaMA-3.1-8B-instruct, Deepseek-Math-7B-instruct, and Qwen2.5-Math-1.5B.

\paragraph{Evaluation.}  
In line with established practices, we evaluate the performance of RL-PLUS on a comprehensive suite of standard mathematical reasoning benchmarks, including GSM8K~\citep{GSM8K}, MATH500~\citep{MATH}, Minerva Math~\citep{MinervaMath}, and OlympiadBench~\citep{OlympiadBench}, as well as on competition-level benchmarks such as AIME 2024~\citep{AIME_AMC} and AMC 2023~\citep{AIME_AMC}.  
Additionally, although our training focuses on math, we extend our evaluation to out-of-domain (OOD) tasks to assess the robustness and generalization capabilities of our approach. The OOD datasets include ARC-c~\citep{clark2018think}(Open-Domain Reasoning), GPQA-diamond~\citep{rein2024gpqa} (Science Graduate Knowledge), MMLU-Pro~\citep{wang2024mmlu} (Reasoning-focused Questions from Academic Exams and Textbooks), as well as three code generation datasets: HumanEval~\citep{codex}, LeetCode~\citep{guo2024deepseek}, and LiveCodeBench~\citep{jain2024livecodebench}. During evaluation, we set the sampling temperature to 0.6 and report the average pass@1 score over 5 runs by default.

\paragraph{Baselines.}  

We compare our approach with two categories of baselines, all trained upon the same base model. The first category comprises eight recently proposed RLVR methods, including: 1) \textbf{SimpleRL}~\citep{Simplerl} and 2) \textbf{OpenReasoner-Zero}~\citep{Open-reasoner-zero} are two open-source RL implementations that train starting from the base model using rule-based rewards. SimpleRL employs a token-level, length-rectified GRPO algorithm, while OpenReasoner-Zero utilizes the PPO algorithm. 3) \textbf{PRIME}~\citep{PRIME_Zero} introduces an implicit process reward based on outcome labels during RL. 4) \textbf{Oat-Zero}~\citep{oat-zero} modifies the GRPO algorithm by removing the standard deviation from the advantage computation and eliminating token-level normalization in the policy loss calculation. 5) \textbf{DAPO}~\citep{DAPO} optimizes GRPO algorithm by introducing four operations: Clip-Higher, Dynamic Sampling, a Token-Level Policy Gradient Loss, and Overlong Reward Shaping. 6) \textbf{LUFFY}~\citep{LUFFY} leverages off-policy reasoning trajectories to augment GRPO. 
7) \textbf{TAPO}~\citep{TAPO} integrates reasoning templates into GRPO sampling process to enhance the model's internal reasoning capabilities. 
8) \textbf{ReLIFT}~\citep{ReLIFT} performs RL and SFT alternately during training.
The second category consists of four straightforward baselines:
1) \textbf{SFT}, supervised fine-tuning using external reasoning trajectory data.
2) \textbf{GRPO}~\citep{DeepSeekMath}, training with GRPO algorithm on question-answer pairs.
3) \textbf{SFT+GRPO}, a common RL cold-start approach that performs SFT before RL training.
4) \textbf{GRPO w/ SFT Loss}, jointly optimizes the GRPO objective and SFT loss during training.

\end{document}